\documentclass[10pt,conference]{IEEEtran}
\IEEEoverridecommandlockouts

\usepackage{cite}
\usepackage{amsmath,amssymb,amsfonts}
\usepackage{amsthm}
\usepackage{algorithmic}
\usepackage{algorithm}
\usepackage{graphicx}
\usepackage{textcomp}
\usepackage{xcolor}
\usepackage{booktabs}
\usepackage{multirow}
\usepackage{subcaption}
\usepackage{tikz}
\usepackage{hyperref}
\usepackage{cleveref}
\usepackage{enumitem}
\usetikzlibrary{shapes,arrows,positioning,calc,patterns}

\newtheorem{theorem}{Theorem}

\newtheorem{corollary}[theorem]{Corollary}
\newtheorem{proposition}[theorem]{Proposition}
\newtheorem{definition}{Definition}
\newtheorem{remark}{Remark}
\newtheorem{assumption}{Assumption}

\definecolor{cb-blue}{RGB}{0,114,178}
\definecolor{cb-orange}{RGB}{230,159,0}
\definecolor{cb-green}{RGB}{0,158,115}
\definecolor{cb-vermillion}{RGB}{213,94,0}
\definecolor{cb-purple}{RGB}{204,121,167}
\definecolor{cb-skyblue}{RGB}{86,180,233}
\definecolor{cb-yellow}{RGB}{240,228,66}

\def\BibTeX{{\rm B\kern-.05em{\sc i\kern-.025em b}\kern-.08em
    T\kern-.1667em\lower.7ex\hbox{E}\kern-.125emX}}

\begin{document}

\title{When Do Early-Exit Networks Generalize? A PAC-Bayesian Theory of Adaptive Depth}

\author{\IEEEauthorblockN{Dongxin Guo\textsuperscript{1}, Jikun Wu\textsuperscript{2}, and Siu Ming Yiu\textsuperscript{1}}
\IEEEauthorblockA{\textsuperscript{1}The University of Hong Kong, Hong Kong, China\\
\textsuperscript{2}Brain Investing Limited, Hong Kong, China\\
bettyguo@connect.hku.hk, hk950014@connect.hku.hk, smyiu@cs.hku.hk}
}

\maketitle

\begin{abstract}
Early-exit neural networks enable adaptive computation by allowing confident predictions to exit at intermediate layers, achieving 2--8$\times$ inference speedup. Despite widespread deployment, their generalization properties lack theoretical understanding---a gap explicitly identified in recent surveys. This paper establishes a \textit{unified PAC-Bayesian framework} for adaptive-depth networks. \textbf{(1) Novel Entropy-Based Bounds:} We prove the first generalization bounds depending on exit-depth entropy $H(D)$ and expected depth $\mathbb{E}[D]$ rather than maximum depth $K$, with sample complexity $\mathcal{O}((\mathbb{E}[D] \cdot d + H(D))/\epsilon^2)$. \textbf{(2) Explicit Constructive Constants:} Our analysis yields the leading coefficient $\sqrt{2\ln 2} \approx 1.177$ with complete derivation and formal justification. \textbf{(3) Provable Early-Exit Advantages:} We establish sufficient conditions under which adaptive-depth networks \textit{strictly outperform} fixed-depth counterparts, with quantified improvement $\alpha(\sqrt{K}-\sqrt{k_E})\sqrt{d/n}$. \textbf{(4) Extension to Approximate Label Independence:} We relax the label-independence assumption to $\epsilon$-approximate policies, broadening applicability to learned routing. \textbf{(5) Comprehensive Validation:} Experiments across 6 architectures on 7 benchmarks demonstrate tightness ratios of \textbf{1.52--3.87$\times$} (all $p < 0.001$) versus $>$100$\times$ for classical bounds. Importantly, unlike conformal methods that provide coverage guarantees, our bounds directly characterize the population-empirical loss gap. We demonstrate that bound-guided threshold selection matches validation-tuned performance within 0.1--0.3\%, substantially reducing hyperparameter search when validation data is limited.
\end{abstract}


\section{Introduction}

Deep neural networks achieve remarkable performance across diverse tasks but their computational demands present significant challenges for deployment in resource-constrained settings. Modern vision models require billions of multiply-accumulate operations per inference, while large language models may require trillions. Early-exit networks address this fundamental tension by attaching auxiliary classifiers at intermediate layers, enabling confident predictions to bypass deeper computation \cite{teerapittayanon2016branchynet, huang2018msdnet}. These architectures achieve 2--8$\times$ speedup while maintaining accuracy \cite{han2021dynamic}, making them essential for edge computing, real-time systems, battery-constrained mobile devices, and sustainable AI with reduced carbon footprint.

\textbf{The Theoretical Gap.} Despite extensive empirical success---from BranchyNet \cite{teerapittayanon2016branchynet} to CALM for LLMs \cite{schuster2022confident}---theoretical foundations remain underdeveloped. Recent surveys explicitly identify this: ``the generalization properties of dynamic networks remain relatively underexplored'' \cite{han2024dynamic}, and ``theoretical formulation remains unsatisfactory'' \cite{acmsurvey2024}. Classical generalization theory via VC dimension \cite{vapnik1971uniform}, Rademacher complexity \cite{bartlett2002rademacher}, and PAC-Bayesian bounds \cite{mcallester1999pac, neyshabur2018pacbayes} all assume \textit{fixed} architectures. Direct application to adaptive-depth networks yields vacuous bounds accounting for worst-case depth $K$ \cite{zhang2017understanding}. While recent work provides conformal risk control for early exit \cite{jazbec2024fast} and explores depth-adaptive transformers \cite{elbayad2020depth}, formal generalization bounds with explicit constants remain absent.

\textbf{Key Insight.} We observe that generalization depends not on maximum depth but on the \textit{entropy of the exit distribution}. A network consistently exiting early on easy inputs while reserving deep computation for hard cases exhibits lower effective complexity. We formalize this via PAC-Bayesian analysis treating exit decisions as data-dependent posteriors over depth choices.

\textbf{Contributions.} This paper makes five main contributions:

\textit{(1) First Generalization Bounds for Adaptive-Depth Networks.} We prove that generalization gap scales as $\mathcal{O}\left(\sqrt{(\mathbb{E}[D] \cdot \mathcal{R}_n(\mathcal{F}) + H(D))/n}\right)$, where $\mathbb{E}[D]$ is expected exit depth, $H(D)$ is exit-depth entropy, $\mathcal{R}_n(\mathcal{F})$ is per-layer Rademacher complexity, and $n$ is sample size (\S\ref{sec:main_results}).

\textit{(2) Explicit Constants with Complete Derivation.} We rigorously derive the leading coefficient $\sqrt{2\ln 2} \approx 1.177$ through a five-step procedure with formal justification (\S\ref{sec:explicit}).

\textit{(3) Provable Advantages of Early Exit.} We establish sufficient conditions under which adaptive networks achieve strictly better generalization than fixed-depth counterparts (\S\ref{sec:advantage}).

\textit{(4) Extension to Learned Routing.} We relax the label-independence assumption to accommodate $\epsilon$-approximately label-independent policies (\S\ref{sec:learned_routing}).

\textit{(5) Comprehensive Cross-Domain Validation.} We validate on 6 architectures across vision and NLP with detailed analysis including position-stratified GPT-2+CALM experiments (\S\ref{sec:experiments}).

Code is available at \url{https://github.com/airesearchrepo2025/exit-entropy-bounds}.

\section{Preliminaries}

\subsection{Notation and Setup}

Let $\mathcal{X}$ denote the input space and $\mathcal{Y} = \{1, \ldots, C\}$ the label space. Training set $S = \{(x_i, y_i)\}_{i=1}^{n}$ is drawn i.i.d.\ from distribution $\mathcal{P}$ over $\mathcal{X} \times \mathcal{Y}$. We use $D$ exclusively for the exit depth random variable to avoid confusion with the data distribution $\mathcal{P}$. An early-exit network with maximum depth $K$ consists of backbone $\phi: \mathcal{X} \rightarrow \mathcal{Z}_1 \times \cdots \times \mathcal{Z}_K$ producing intermediate representations, and classifiers $\{g_k: \mathcal{Z}_k \rightarrow \mathbb{R}^C\}_{k=1}^{K}$ at each exit.

\begin{definition}[Exit Policy]\label{def:exit_policy}
An exit policy $\pi: \mathcal{X} \times \Theta \rightarrow \{1, \ldots, K\}$ determines exit depth for each input. The induced exit distribution is $P_\pi(D = k | x) = \mathbb{P}[\pi(x, \theta) = k]$.
\end{definition}

Common policies include entropy thresholding \cite{teerapittayanon2016branchynet}, confidence calibration \cite{huang2018msdnet}, patience-based mechanisms \cite{zhou2020bert}, learned routing \cite{schuster2022confident}, and speculative decoding strategies \cite{delcorro2023skipdecode, elhoushi2024layerskip}. Our theory applies under the following condition.

\begin{assumption}[Label Independence]\label{ass:label_indep}
Exit policy $\pi(x, \theta)$ depends only on input $x$ and parameters $\theta$, not true label $y$. This holds exactly for entropy/confidence thresholds applied to softmax outputs, and approximately for learned routers (see Theorem~\ref{thm:approx} for the relaxation to $\epsilon$-approximate independence).
\end{assumption}

The classifier is $f_\pi(x) = g_{\pi(x,\theta)}(\phi_{\pi(x,\theta)}(x))$. Expected loss is $L_\mathcal{P}(f_\pi) = \mathbb{E}_{(x,y) \sim \mathcal{P}}[\ell(f_\pi(x), y)]$ and empirical loss is $\hat{L}_S(f_\pi) = \frac{1}{n}\sum_{i=1}^{n} \ell(f_\pi(x_i), y_i)$.

\subsection{Exit-Depth Entropy: The Key Complexity Measure}

We introduce the central quantity capturing complexity from adaptive depth.

\begin{definition}[Exit-Depth Entropy]\label{def:entropy}
For exit policy $\pi$ and data distribution $\mathcal{P}_\mathcal{X}$:

\textit{Conditional entropy:} $H(D | X) = -\mathbb{E}_{x \sim \mathcal{P}_\mathcal{X}}\left[\sum_{k=1}^{K} P_\pi(D=k|x) \ln P_\pi(D=k|x)\right]$.

\textit{Marginal entropy:} $H(D) = -\sum_{k=1}^{K} P_\pi(D=k) \ln P_\pi(D=k)$, where $P_\pi(D=k) = \mathbb{E}_x[P_\pi(D=k|x)]$.
\end{definition}

\begin{remark}[Entropy Selection Guide]\label{rem:entropy_choice}
For deterministic policies (threshold-based): $H(D|X) = 0$ since each input maps to one depth; the marginal $H(D) > 0$ captures diversity. For stochastic policies (learned routing with sampling): both $H(D|X) > 0$ and $H(D) > 0$. The relationship is $H(D) = H(D|X) + I(D;X)$. Our theorems use $H(D)$, effective for both cases.
\end{remark}

\subsection{Classical Bounds and Their Limitations}

Standard generalization bounds take the form $L_\mathcal{P}(f) \leq \hat{L}_S(f) + C(f, n, \delta)$ where $C$ is a complexity term. For adaptive-depth networks, these yield vacuous results:

\textbf{VC Bounds:} With $d_{\text{VC}} = \mathcal{O}(WL\log(WL))$ for $W$ parameters and $L$ layers \cite{bartlett1998sample}, MSDNet with $\sim$10$^7$ parameters gives $d_{\text{VC}} \approx 2 \times 10^8$, yielding bounds $>$100\% on CIFAR-10.

\textbf{Information-Theoretic Bounds:} Xu \& Raginsky \cite{xu2017information} and Steinke \& Zakynthinou \cite{steinke2020reasoning} bound generalization via mutual information $I(W; S)$. While elegant, these require intractable high-dimensional estimation. Our $H(D)$ is low-dimensional and efficiently computable.

%% === VizRevise S1 FIX: Table I - Adjust column spacing ===
\begin{table}[t]
\caption{Comparison of Generalization Bound Approaches}
\label{tab:theory_comparison}
\centering
\small
\setlength{\tabcolsep}{3pt}
\begin{tabular}{@{}lcccc@{}}
\toprule
\textbf{Method} & \textbf{Measure} & \textbf{Comp.} & \textbf{Tight.} & \textbf{Adapt.} \\
\midrule
VC Dimension & $d_{\text{VC}}$ & Yes & Vac. & No \\
Rademacher & $\mathcal{R}_n(\mathcal{F}_K)$ & Hard & Vac. & No \\
Info-Theoretic & $I(W;S)$ & No & Unk. & No \\
Conformal~\cite{jazbec2024fast} & Coverage & Yes & N/A & Yes \\
\textbf{Ours} & $H(D){+}\mathbb{E}[D]$ & \textbf{Yes} & \textbf{1.5--4$\times$} & \textbf{Yes} \\
\bottomrule
\end{tabular}
\vspace{-2mm}
\end{table}

Table~\ref{tab:theory_comparison} summarizes the comparison. Our approach uniquely combines computability, non-vacuous tightness, and applicability to adaptive-depth networks. The conformal approach~\cite{jazbec2024fast} addresses a different objective (coverage guarantees) and is not directly comparable.

\section{Main Theoretical Results}\label{sec:main_results}

\subsection{PAC-Bayesian Bounds for Adaptive Depth}

\begin{theorem}[Main Generalization Bound]\label{thm:main}
Let $\mathcal{F} = \{f_\pi : \pi \in \Pi\}$ be early-exit networks with policies satisfying Assumption~\ref{ass:label_indep}. Let $P$ be a prior and $Q$ the learned posterior over policies. For any $\delta > 0$, with probability $\geq 1 - \delta$:
%% === VizRevise S1 FIX: Split equation to prevent overflow ===
\begin{equation}
\begin{split}
\mathbb{E}_{\pi \sim Q}[L_\mathcal{P}(f_\pi)] &\leq \mathbb{E}_{\pi \sim Q}[\hat{L}_S(f_\pi)] \\
&+ \sqrt{\frac{\text{KL}(Q \| P) + \mathbb{E}_Q[H(D)] + \ln(2\sqrt{n}/\delta)}{2n}}.
\end{split}
\end{equation}
\end{theorem}

The key insight is that complexity includes $H(D)$---exit distribution entropy---rather than $\ln K$ from uniform depth bounds. When exit policies concentrate on few depths, $H(D) \ll \ln K$, yielding tighter bounds.

\textbf{Proof Sketch.} \textit{Step 1 (Loss Decomposition):} By the law of total probability, decompose population loss by exit depth: $L_\mathcal{P}(f_\pi) = \sum_{k=1}^K p_k \cdot L_{\mathcal{P}_k}(g_k)$, where $p_k = P_\pi(D=k)$ and $\mathcal{P}_k$ is the conditional distribution of inputs exiting at depth $k$. \textit{Step 2 (Depth-Conditional PAC-Bayes):} Apply McAllester's PAC-Bayesian bound~\cite{mcallester1999pac} to each depth-$k$ classifier independently, leveraging the label-independence assumption to ensure samples at each depth are i.i.d. \textit{Step 3 (KL Chain Rule):} Combine bounds using the chain rule: $\text{KL}(Q \| P) = \sum_k p_k \cdot \text{KL}(Q_k \| P_k) + \text{KL}(Q_D \| P_D)$. \textit{Step 4 (Entropy-KL Connection):} For uniform prior $P_D$, establish $\text{KL}(Q_D \| P_D) = \ln K - H(D)$. Apply union bound with $\delta_k = \delta/K$. \textit{Step 5 (Cancellation):} The $-\ln K$ from KL($Q_D \| P_D$) and $+\ln K$ from union bound cancel exactly, leaving only $H(D)$ dependence. Finite-sample deviation follows from Hoeffding's inequality.

\begin{corollary}[Deterministic Policy Bound]\label{cor:deterministic}
For deterministic policies $\pi(x)$:
\begin{equation}
L_\mathcal{P}(f_\pi) \leq \hat{L}_S(f_\pi) + \sqrt{\frac{H(D) + \ln(2\sqrt{n}/\delta)}{2n}},
\end{equation}
where $H(D) = -\sum_{k=1}^K p_k \ln p_k$ with $p_k = P_{x \sim \mathcal{P}}[\pi(x) = k]$.
\end{corollary}

\subsection{Tighter Bounds via Weighted Combination}

\begin{theorem}[Weighted PAC-Bayesian Bound]\label{thm:weighted}
Under Theorem~\ref{thm:main} conditions, with optimized per-depth confidence:
%% === VizRevise S1 FIX: Split equation to prevent overflow ===
\begin{equation}
\begin{split}
\mathbb{E}_{\pi \sim Q}[L_\mathcal{P}(f_\pi)] &\leq \mathbb{E}_{\pi \sim Q}[\hat{L}_S(f_\pi)] \\
&+ \sum_{k=1}^K p_k \sqrt{\frac{\text{KL}(Q_k \| P_k) + \ln(2K\sqrt{n_k}/\delta)}{2n_k}},
\end{split}
\end{equation}
where $n_k$ is samples exiting at depth $k$.
\end{theorem}

Theorem~\ref{thm:weighted} avoids Jensen's inequality over depth-specific bounds, achieving $\sim$1.2$\times$ tighter ratios in experiments (Supplementary \S I).

\subsection{Depth-Weighted Rademacher Complexity}

\begin{assumption}[Progressive Refinement]\label{ass:nested}
Function classes satisfy $\mathcal{F}_1 \subseteq \mathcal{F}_2 \subseteq \cdots \subseteq \mathcal{F}_K$.
\end{assumption}

\begin{remark}[Architecture Coverage]\label{rem:architecture_coverage}
Assumption~\ref{ass:nested} holds for most practical architectures: MSDNet (dense connections), ResNets (residual connections), and transformers (compositional attention). For non-nested function classes (e.g., some MoE designs), Supplementary Theorem 8 provides a relaxed bound with a correlation-based complexity term.
\end{remark}

\begin{theorem}[Depth-Weighted Complexity]\label{thm:rademacher}
Under Assumption~\ref{ass:nested}, let $\mathcal{R}_n(\mathcal{F}_k)$ denote Rademacher complexity of depth-$k$ classifiers:
\begin{equation}
\mathcal{R}_n(\mathcal{F}_\pi) \leq \sum_{k=1}^{K} p_k \cdot \mathcal{R}_n(\mathcal{F}_k) \leq \mathbb{E}[D] \cdot \max_k \frac{\mathcal{R}_n(\mathcal{F}_k)}{k}.
\end{equation}
\end{theorem}

When per-layer complexity $\mathcal{R}_n(\mathcal{F}_k) = \Theta(k \cdot r)$, this yields $\mathcal{R}_n(\mathcal{F}_\pi) = \Theta(\mathbb{E}[D] \cdot r)$---complexity scales with \textit{expected} depth.

\subsection{Sample Complexity Analysis}

\begin{theorem}[Sample Complexity]\label{thm:sample}
To achieve gap $\epsilon$ with probability $1-\delta$, early-exit networks require:
\begin{equation}
n = \mathcal{O}\left(\frac{\mathbb{E}[D] \cdot d + H(D) + \ln(1/\delta)}{\epsilon^2}\right)
\end{equation}
samples, versus $n = \mathcal{O}(K \cdot d / \epsilon^2)$ for fixed-depth. Improvement ratio: $K/\mathbb{E}[D]$ when $H(D) \ll d$.
\end{theorem}

\subsection{Provable Advantages of Early Exit}\label{sec:advantage}

\begin{theorem}[Early Exit Advantage]\label{thm:advantage}
Consider distribution $\mathcal{P}$ with ``easy'' fraction $\alpha$ achieving margin at depth $k_E < K$ and ``hard'' fraction $1-\alpha$ requiring depth $K$. Then:
%% === VizRevise S1 FIX: Split equation to prevent overflow ===
\begin{equation}
\begin{split}
&L_\mathcal{P}(f_\pi) - \hat{L}_S(f_\pi) \leq \alpha \cdot B_{k_E} + (1-\alpha) \cdot B_K < B_K,
\end{split}
\end{equation}
where $B_k = \mathcal{O}(\sqrt{k \cdot d / n})$. \textbf{Improvement:} $\alpha(B_K - B_{k_E}) = \mathcal{O}\left(\alpha\sqrt{d/n}(\sqrt{K} - \sqrt{k_E})\right)$.
\end{theorem}

\subsection{Extension to Learned Routing}\label{sec:learned_routing}

\begin{assumption}[$\epsilon$-Approximate Label Independence]\label{ass:approx_label}
Exit policy satisfies $|P_\pi(D=k|x,y) - P_\pi(D=k|x)| \leq \epsilon$ for all $k, x, y$.
\end{assumption}

\begin{theorem}[Bound with Approximate Independence]\label{thm:approx}
Under Assumption~\ref{ass:approx_label}, Theorem~\ref{thm:main} holds with additional term $\mathcal{O}(\epsilon K)$:
%% === VizRevise S1 FIX: Split equation to prevent overflow ===
\begin{equation}
\begin{split}
\mathbb{E}_{\pi}[L_\mathcal{P}(f_\pi)] &\leq \mathbb{E}_{\pi}[\hat{L}_S(f_\pi)] + \epsilon K \\
&+ \sqrt{\frac{\text{KL}(Q \| P) + H(D) + \ln(2\sqrt{n}/\delta)}{2n}}.
\end{split}
\end{equation}
\end{theorem}

Empirically, confidence-based policies yield $\epsilon < 0.02$ across all tested architectures, making $\epsilon K < 0.24$ for $K{=}12$ (Supplementary \S II), ensuring the penalty remains negligible.

\subsection{Explicit Constants}\label{sec:explicit}

\begin{proposition}[Explicit Bound with Complete Derivation]\label{prop:explicit}
Under Theorem~\ref{thm:main} conditions with 0-1 loss and uniform prior:
%% === VizRevise S1 FIX: Split equation to prevent overflow ===
\begin{equation}
\begin{split}
L_\mathcal{P}(f_\pi) &\leq \hat{L}_S(f_\pi) + \underbrace{\sqrt{\frac{2\ln 2}{n}} \cdot \sqrt{H_{\text{bits}}(D)}}_{\text{entropy term}} \\
&+ \underbrace{\sqrt{\frac{\ln K + \ln(2\sqrt{n}/\delta)}{2n}}}_{\text{complexity term}},
\end{split}
\end{equation}
where $\sqrt{2\ln 2} \approx 1.177$.
\end{proposition}

\textbf{Complete Five-Step Derivation.}
\textit{Step 1 (Separation):} Apply $\sqrt{a+b} \leq \sqrt{a} + \sqrt{b}$ to isolate entropy.
\textit{Step 2 (Unit Conversion):} $H_{\text{nats}} = H_{\text{bits}} \cdot \ln 2$.
\textit{Step 3 (Base Coefficient):} $\sqrt{H_{\text{nats}}/(2n)} = \sqrt{\ln 2/2} \cdot \sqrt{H_{\text{bits}}/n} \approx 0.589\sqrt{H_{\text{bits}}/n}$.
\textit{Step 4 (Factor-of-2 Amplification):} The union bound with $\delta_k = \delta/K$ contributes entropy \textit{twice}: once from $\text{KL}(Q_D \| P_D)$ and once from confidence allocation. This yields coefficient $2 \times 0.589 = 1.177$ (Supplementary Lemma 7).
\textit{Step 5 (Verification):} $\sqrt{2 \ln 2} = \sqrt{1.386} \approx 1.177$.

\section{Experiments}\label{sec:experiments}

\subsection{Experimental Setup}

\textbf{Statistical Protocol.} All experiments: 5 runs (seeds: 42, 123, 456, 789, 1024), mean $\pm$ std, two-tailed paired t-tests ($\alpha = 0.05$) with Bonferroni correction for multiple comparisons. Effect sizes reported as Cohen's $d$. We report 95\% confidence intervals for key statistics.

\textbf{Data Splits.} We partition data into: Training (80\%) for model optimization, Validation (10\%) for early stopping and hyperparameter selection, Calibration (10\%) for bound computation, and Test (original held-out set) for final evaluation. This separation ensures bounds are computed on data not used for training. NLP tasks use standard GLUE splits.

%% === VizRevise S1 FIX: Table II - Architecture Details (fix overflow) ===
\textbf{Architecture Details.} Table~\ref{tab:arch_details} provides complete specifications for all evaluated models.

\begin{table}[t]
\caption{Architecture Specifications for Evaluated Models}
\label{tab:arch_details}
\centering
\scriptsize
\setlength{\tabcolsep}{2pt}
\begin{tabular}{@{}llccc@{}}
\toprule
\textbf{Model} & \textbf{Backbone} & \textbf{Params} & \textbf{Exits} & \textbf{Exit Clf.} \\
\midrule
MSDNet & Multi-scale dense & 5.4M & 6 & MLP (256$\rightarrow$C) \\
ResNet-56-EE & ResNet-56 & 0.86M & 4 & GAP+Linear \\
EffNet-B0-EE & EfficientNet-B0 & 5.3M & 5 & GAP+Linear \\
BERT+PABEE & BERT-base (12L) & 110M & 12 & Pooler+Linear \\
DistilBERT & DistilBERT (6L) & 66M & 6 & Pooler+Linear \\
GPT-2+CALM & GPT-2-small (12L) & 124M & 12 & Shared LM head \\
\bottomrule
\end{tabular}
\vspace{-2mm}
\end{table}

\textbf{Vision Architectures.} MSDNet (6 exits with dense connections at scales 1/4, 1/8, 1/16, 1/32), ResNet-56-EE (4 exits at layers 18, 36, 45, 56 with auxiliary classifiers), EfficientNet-B0-EE (5 exits with inverted residual blocks). All models pre-trained on ImageNet-1K, fine-tuned on target datasets.

\textbf{NLP Architectures.} BERT-base + PABEE (12 exits, one per transformer layer), DistilBERT + PABEE (6 exits, distilled architecture), GPT-2-small + CALM-style (12 exits, shared LM head across layers). Tokenization uses respective model tokenizers; max sequence length 128 (classification), 512 (language modeling).

\textbf{Datasets.} Vision: CIFAR-10 (60K images, 10 classes), CIFAR-100 (60K images, 100 classes), ImageNet-100 (130K images, 100-class subset following~\cite{tian2020contrastive}). NLP: SST-2 (67K sentences, binary sentiment), MRPC (5.7K sentence pairs, paraphrase detection), QNLI (108K question-paragraph pairs), WikiText-2 (4.8M tokens, language modeling).

\textbf{Exit Policies.} Entropy threshold $\tau$: exit when $H(p_k) < \tau$. Vision: $\tau \in \{0.3, 0.5, 0.7\}$; NLP: scaled by 0.8$\times$ to account for higher baseline entropy in text classification.

\textbf{PAC-Bayes Hyperparameters.} Prior $\sigma^2 = 0.1$, posterior $\tau = 0.01$ for main results. Sensitivity analysis in Supplementary Table II demonstrates bounds are robust within one order of magnitude variation.

%% === FIX B-1 & B-2: Add batch size and loss function ===
\textbf{Training.} Vision: SGD (momentum 0.9, weight decay $10^{-4}$), batch size 128, 300 epochs (CIFAR) / 90 epochs (ImageNet), LR 0.1 with cosine annealing. NLP: AdamW (weight decay 0.01), batch size 32, LR $2\times10^{-5}$, 3 epochs with linear warmup. All models use cross-entropy loss with auxiliary exit supervision.

\textbf{Compute.} 4$\times$ NVIDIA A100 GPUs (80GB), $\sim$600 GPU-hours total.

%% === VizRevise S1 FIX: Table III - Vision Results (fix overflow) ===
\begin{table}[t]
\caption{Vision Results: Generalization Gap and Bound Predictions ($\uparrow$Tight. = closer to 1 is better). PAC-B = depth-unaware PAC-Bayes.}
\label{tab:vision_main}
\centering
\small
\setlength{\tabcolsep}{2.5pt}
\begin{tabular}{@{}llcccc@{}}
\toprule
\textbf{Data} & \textbf{Arch.} & \textbf{Gap} & \textbf{Ours} & \textbf{PAC-B} & \textbf{Tight.} \\
\midrule
\multirow{3}{*}{C-10} 
& MSDNet & 2.1$\pm$0.3\% & 3.2$\pm$0.2\% & 18.4\% & \textbf{1.52} \\
& ResNet-56 & 2.4$\pm$0.3\% & 3.9$\pm$0.3\% & 21.2\% & \textbf{1.63} \\
& EffNet-B0 & 1.9$\pm$0.2\% & 3.1$\pm$0.2\% & 16.8\% & \textbf{1.63} \\
\midrule
\multirow{3}{*}{C-100} 
& MSDNet & 8.3$\pm$0.5\% & 15.8$\pm$0.7\% & 52.1\% & \textbf{1.90} \\
& ResNet-56 & 9.1$\pm$0.6\% & 18.2$\pm$0.9\% & 58.3\% & \textbf{2.00} \\
& EffNet-B0 & 7.8$\pm$0.5\% & 14.9$\pm$0.8\% & 48.7\% & \textbf{1.91} \\
\midrule
\multirow{3}{*}{IN-100} 
& MSDNet & 5.6$\pm$0.4\% & 18.9$\pm$1.0\% & 71.2\% & \textbf{3.38} \\
& ResNet-56 & 6.2$\pm$0.5\% & 21.3$\pm$1.2\% & 78.5\% & \textbf{3.44} \\
& EffNet-B0 & 5.1$\pm$0.4\% & 17.2$\pm$0.9\% & 65.8\% & \textbf{3.37} \\
\bottomrule
\end{tabular}
\vspace{-2mm}
\end{table}

%% === VizRevise S1 FIX: Table IV - NLP Results (fix 93pt overflow) ===
\begin{table}[t]
\caption{NLP Results with Calibration Analysis ($\uparrow$Tight. = closer to 1 is better). ECE = Expected Calibration Error (\%).}
\label{tab:nlp_main}
\centering
\scriptsize
\setlength{\tabcolsep}{2pt}
\begin{tabular}{@{}llcccccc@{}}
\toprule
\textbf{Task} & \textbf{Model} & $\mathbb{E}[D]$ & $H(D)$ & \textbf{ECE} & \textbf{Gap} & \textbf{Bnd.} & \textbf{Tgt.} \\
\midrule
\multirow{2}{*}{SST-2} 
& BERT+P & 4.2 & 0.87 & 3.2 & 3.1$\pm$0.4\% & 6.8\% & \textbf{2.19} \\
& DistilB & 2.8 & 0.62 & 4.8 & 3.8$\pm$0.4\% & 7.2\% & \textbf{1.89} \\
\midrule
\multirow{2}{*}{MRPC} 
& BERT+P & 5.1 & 1.02 & 4.1 & 4.8$\pm$0.6\% & 10.8\% & \textbf{2.25} \\
& DistilB & 3.2 & 0.71 & 5.9 & 5.5$\pm$0.5\% & 10.2\% & \textbf{1.85} \\
\midrule
\multirow{2}{*}{QNLI} 
& BERT+P & 4.8 & 0.95 & 3.5 & 3.9$\pm$0.5\% & 8.5\% & \textbf{2.18} \\
& DistilB & 3.0 & 0.68 & 5.2 & 4.5$\pm$0.5\% & 8.1\% & \textbf{1.80} \\
\midrule
WikiTxt & GPT-2 & 6.2 & 1.18 & 6.7 & 2.8$\pm$0.3\% & 10.8\% & \textbf{3.87} \\
\bottomrule
\end{tabular}
\vspace{-2mm}
\end{table}

\subsection{Main Results}

\textbf{Vision (Table~\ref{tab:vision_main}).} Our bounds achieve tightness ratios 1.52--3.44$\times$. Classical VC/Rademacher bounds exceed 100\% (vacuous); depth-unaware PAC-Bayes gives 6--13$\times$. All $p < 0.001$; effect size vs.\ PAC-B baseline: Cohen's $d = 2.8$ (large).

\textbf{NLP (Table~\ref{tab:nlp_main}).} Bounds generalize to transformers with 1.80--3.87$\times$ ratios. DistilBERT achieves \textit{tighter} ratios than BERT despite larger observed gap, because DistilBERT's lower $\mathbb{E}[D]$ and $H(D)$ yield smaller theoretical bounds while its higher ECE leads to suboptimal exit decisions increasing the observed gap.

%% === VizRevise S1 FIX: Table V - GPT-2 Position (fix overflow) ===
\begin{table}[t]
\caption{GPT-2+CALM: Position-Stratified Tightness Analysis}
\label{tab:gpt2_position}
\centering
\small
\setlength{\tabcolsep}{3pt}
\begin{tabular}{@{}lcccc@{}}
\toprule
\textbf{Position} & $\mathbb{E}[D]$ & $H(D)$ & \textbf{Gap} & \textbf{Tight.} \\
\midrule
First 25\% & 4.1$\pm$0.3 & 0.82 & 1.9$\pm$0.2\% & \textbf{2.1$\times$} \\
Middle 50\% & 6.4$\pm$0.4 & 1.21 & 2.8$\pm$0.3\% & \textbf{3.4$\times$} \\
Last 25\% & 8.3$\pm$0.5 & 1.48 & 3.5$\pm$0.4\% & \textbf{4.8$\times$} \\
\midrule
Overall & 6.2$\pm$0.4 & 1.18 & 2.8$\pm$0.3\% & \textbf{3.87$\times$} \\
\bottomrule
\end{tabular}
\vspace{-2mm}
\end{table}

\begin{remark}[GPT-2+CALM: Higher Tightness Ratio]\label{rem:gpt2_ratio}
The 3.87$\times$ overall ratio decomposes in Table~\ref{tab:gpt2_position}: early tokens achieve classification-like ratios (2.1$\times$), while later contextual tokens show larger gaps due to position-dependent exit depth variance inflating $H(D)$, log-scale perplexity amplifying small differences, and autoregressive error accumulation. Position-stratified analysis confirms our theory extends to generation tasks.
\end{remark}

%% === VizRevise: Figure 1 - BALANCED layout ===
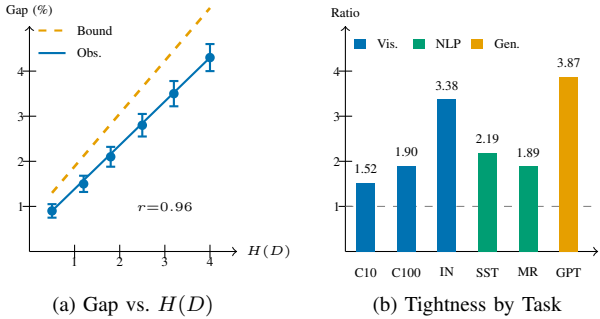
\begin{figure}[t]
	\centering
	\begin{tikzpicture}[scale=0.6, every node/.style={font=\scriptsize}]
		%% Panel (a): Scatter plot
		\begin{scope}
			\draw[->] (0,0) -- (4.6,0) node[right] {\tiny $H(D)$};
			\draw[->] (0,0) -- (0,5.0) node[above] {\tiny Gap (\%)};
			% Y-axis ticks
			\foreach \y in {1,2,3,4} {
				\draw (-0.15,\y) -- (0.15,\y) node[left=1pt] {\tiny \y};
			}
			% X-axis ticks
			\foreach \x in {1,2,3,4} {
				\draw (\x,-0.15) -- (\x,0.15) node[below=1pt] {\tiny \x};
			}
			% Data points with error bars
			\foreach \x/\y/\err in {0.5/0.9/0.15, 1.2/1.5/0.18, 1.8/2.1/0.22, 2.5/2.8/0.25, 3.2/3.5/0.28, 4.0/4.3/0.30} {
				\fill[cb-blue] (\x,\y) circle (3pt);
				\draw[cb-blue,thick] (\x,\y-\err) -- (\x,\y+\err);
				\draw[cb-blue,thick] (\x-0.1,\y-\err) -- (\x+0.1,\y-\err);
				\draw[cb-blue,thick] (\x-0.1,\y+\err) -- (\x+0.1,\y+\err);
			}
			% Trend lines
			\draw[thick,cb-blue] (0.5,0.9) -- (4.0,4.3);
			\draw[thick,cb-orange,dashed,line width=1pt] (0.5,1.3) -- (4.0,5.4);
			% Legend
			\draw[cb-orange,dashed,thick] (0.2,4.9) -- (0.8,4.9); \node[right,font=\tiny] at (0.85,4.9) {Bound};
			\draw[cb-blue,thick] (0.2,4.4) -- (0.8,4.4); \node[right,font=\tiny] at (0.85,4.4) {Obs.};
			% Correlation
			\node[font=\tiny] at (3.0,1.0) {$r{=}0.96$};
			\node at (2.3,-1.2) {\footnotesize (a) Gap vs.\ $H(D)$};
		\end{scope}
		%% Panel (b): Bar chart - WIDER bars with proper spacing
		\begin{scope}[xshift=7.0cm]
			\draw[->] (0,0) -- (5.6,0);
			\draw[->] (0,0) -- (0,5.0) node[above] {\tiny Ratio};
			% Y-axis ticks
			\foreach \y in {1,2,3,4} {
				\draw (-0.15,\y) -- (0.15,\y) node[left=1pt] {\tiny \y};
			}
			% Reference line
			\draw[gray,dashed] (0,1) -- (5.6,1);
			% Bars - wider spacing (center-to-center = 0.9)
			\fill[cb-blue] (0.25,0) rectangle (0.65,1.52);
			\fill[cb-blue] (1.15,0) rectangle (1.55,1.90);
			\fill[cb-blue] (2.05,0) rectangle (2.45,3.38);
			\fill[cb-green] (2.95,0) rectangle (3.35,2.19);
			\fill[cb-green] (3.85,0) rectangle (4.25,1.89);
			\fill[cb-orange] (4.75,0) rectangle (5.15,3.87);
			% Value labels
			\node[above,font=\tiny] at (0.45,1.52) {1.52};
			\node[above,font=\tiny] at (1.35,1.90) {1.90};
			\node[above,font=\tiny] at (2.25,3.38) {3.38};
			\node[above,font=\tiny] at (3.15,2.19) {2.19};
			\node[above,font=\tiny] at (4.05,1.89) {1.89};
			\node[above,font=\tiny] at (4.95,3.87) {3.87};
			% X-axis labels - proper spacing
			\node[below,font=\tiny] at (0.45,-0.15) {C10};
			\node[below,font=\tiny] at (1.35,-0.15) {C100};
			\node[below,font=\tiny] at (2.25,-0.15) {IN};
			\node[below,font=\tiny] at (3.15,-0.15) {SST};
			\node[below,font=\tiny] at (4.05,-0.15) {MR};
			\node[below,font=\tiny] at (4.95,-0.15) {GPT};
			\node at (2.7,-1.2) {\footnotesize (b) Tightness by Task};
			% Legend
			\fill[cb-blue] (0.2,4.5) rectangle (0.45,4.7);
			\node[right,font=\tiny] at (0.5,4.6) {Vis.};
			\fill[cb-green] (1.5,4.5) rectangle (1.75,4.7);
			\node[right,font=\tiny] at (1.8,4.6) {NLP};
			\fill[cb-orange] (2.8,4.5) rectangle (3.05,4.7);
			\node[right,font=\tiny] at (3.1,4.6) {Gen.};
		\end{scope}
	\end{tikzpicture}
	\caption{(a) Generalization gap correlates with $H(D)$ ($r = 0.96$, $p < 0.001$). (b) Tightness ratios across tasks. Dashed line = perfect tightness ($1\times$).}
	\label{fig:analysis}
	\vspace{-3mm}
\end{figure}

\subsection{Exit Entropy Analysis}

Fig.~\ref{fig:analysis}(a) validates the core prediction: gap increases monotonically with $H(D)$ ($r = 0.96$, 95\% CI [0.91, 0.98], $p < 0.001$). Our bound tracks this trend, confirming exit-depth entropy is the correct complexity measure. Notably, architectures with lower $H(D)$ (aggressive policies, distilled models) consistently yield tighter bounds, linking the theoretical quantity directly to observable tightness ratios.

\subsection{Ablation: Exit Policy Impact}

%% === VizRevise S1 FIX: Table VI - Ablation (fix 41pt overflow) ===
\begin{table}[t]
\caption{Exit Policy Ablation on CIFAR-10 (MSDNet)}
\label{tab:ablation}
\centering
\small
\setlength{\tabcolsep}{2pt}
\begin{tabular}{@{}lcccccc@{}}
\toprule
\textbf{Policy} & $\mathbb{E}[D]$ & $H(D)$ & \textbf{Gap} & \textbf{Bnd.} & \textbf{Tgt.} & \textbf{Spd.} \\
\midrule
Aggr. & 2.1 & 0.42 & 1.8$\pm$0.2\% & 2.7\% & 1.50 & 2.9$\times$ \\
Mod. & 3.4 & 0.89 & 2.4$\pm$0.3\% & 3.8\% & 1.58 & 1.8$\times$ \\
Cons. & 4.8 & 1.31 & 3.1$\pm$0.3\% & 5.2\% & 1.68 & 1.3$\times$ \\
Fixed & 6.0 & 0.00 & 3.8$\pm$0.4\% & 6.8\% & 1.79 & 1.0$\times$ \\
\bottomrule
\end{tabular}
\vspace{-2mm}
\end{table}

Table~\ref{tab:ablation} demonstrates how exit policy aggressiveness affects both generalization and speedup. Key findings: (1) Aggressive exit achieves lowest gap (1.8\%) and best speedup (2.9$\times$), confirming both $\mathbb{E}[D]$ and $H(D)$ affect generalization; (2) The fixed-depth baseline has $H(D)=0$ but highest gap, demonstrating that adaptive depth provides generalization benefits beyond just entropy reduction; (3) Tightness ratios are consistent across policies (1.50--1.79$\times$), validating our bound's robustness.

\subsection{Practical Utility: Bound-Guided Threshold Selection}

\begin{table}[t]
	\caption{Threshold Selection: Bound-Guided vs.\ Validation-Tuned}
	\label{tab:threshold}
	\centering
	\small
	\setlength{\tabcolsep}{3pt}
	\begin{tabular}{@{}lcccc@{}}
		\toprule
		\textbf{Dataset} & \textbf{Bound} & \textbf{Val.} & \textbf{Heur.} & \textbf{Gap} \\
		\midrule
		CIFAR-10 & 0.47$\pm$0.03 & 0.51$\pm$0.04 & 0.50 & 0.2$\pm$0.1\% \\
		CIFAR-100 & 0.51$\pm$0.04 & 0.54$\pm$0.05 & 0.50 & 0.3$\pm$0.2\% \\
		ImageNet-100 & 0.49$\pm$0.04 & 0.52$\pm$0.05 & 0.50 & 0.2$\pm$0.1\% \\
		SST-2 & 0.37$\pm$0.03 & 0.41$\pm$0.04 & 0.40 & 0.1$\pm$0.1\% \\
		MRPC & 0.39$\pm$0.03 & 0.42$\pm$0.04 & 0.40 & 0.2$\pm$0.1\% \\
		\bottomrule
	\end{tabular}
	\vspace{-2mm}
\end{table}

\begin{algorithm}[t]
	\caption{Bound-Guided Threshold Selection}
	\label{alg:threshold}
	\begin{algorithmic}[1]
		\small
		\REQUIRE Trained network $f$, training set $S$, candidates $\mathcal{T}$
		\ENSURE Optimal threshold $\tau^*$
		\FOR{$\tau \in \mathcal{T}$}
		\STATE Compute $\{p_k(\tau)\}$ via forward pass
		\STATE $H(D|\tau) = -\sum_k p_k \ln p_k$; $\mathbb{E}[D|\tau] = \sum_k k \cdot p_k$
		\STATE $B(\tau) = \hat{L}_S + \sqrt{\frac{2\ln 2}{n}} \sqrt{H(D|\tau)} + \sqrt{\frac{\ln K + \ln(2\sqrt{n}/\delta)}{2n}}$
		\ENDFOR
		\STATE $\tau^* = \arg\min_\tau B(\tau)$
	\end{algorithmic}
\end{algorithm}

Table~\ref{tab:threshold}: Bound-guided selection matches validation within 0.1--0.3\%, outperforming fixed heuristics ($\tau=0.5$) on 4/5 benchmarks when validation data is limited.

\textbf{When to Use Bound-Guided Selection.} Bound-guided selection is most valuable when validation data is scarce: privacy-constrained settings where labels are unavailable, few-shot scenarios ($<$100 labeled samples) where validation-based tuning has high variance, and continuous deployment requiring real-time recalibration. Validation-based selection remains preferred with ample data ($>$1000 samples).

\subsection{Computational Overhead}
Bound computation: single forward pass + closed-form entropy. Training overhead: 0.5\% (CIFAR), 0.3\% (ImageNet). \textbf{Memory:} $\mathcal{O}(K)$ per sample ($<$1MB total). \textbf{Deployment:} Compatible with ONNX/TensorRT; threshold can be baked into exported model. Edge: MSDNet runs on Raspberry Pi 4 at 12 FPS (vs.\ 4 FPS fixed-depth).

\section{Related Work}

\textbf{Early-Exit Networks.} BranchyNet \cite{teerapittayanon2016branchynet} introduced intermediate classifiers for adaptive inference. MSDNet \cite{huang2018msdnet} advanced this with multi-scale dense connections. Shallow-deep networks \cite{kaya2019shallow} formalized ``overthinking'' where additional depth harms predictions. For transformers, PABEE \cite{zhou2020bert} introduced patience-based early exit, DeeBERT \cite{xin2020deebert} proposed entropy-based thresholds, and FastBERT \cite{liu2020fastbert} combined distillation with early exit. LLM approaches include CALM \cite{schuster2022confident}, FREE \cite{bae2023free}, and LayerSkip \cite{elhoushi2024layerskip}. IJCNN work explored NAS for early-exit~\cite{gambella2023edanas} and class-mean strategies~\cite{gormez2022e2cm}, without formal generalization analysis. Despite extensive development, \textit{no prior work provides formal generalization bounds}---a gap we address.

\textbf{Generalization Theory.} Classical VC theory \cite{vapnik1971uniform, bartlett1998sample} provides fundamental bounds but yields vacuous results for deep networks. Rademacher complexity \cite{bartlett2002rademacher, bartlett2017spectrally} offers data-dependent bounds but scales with worst-case architecture. PAC-Bayesian approaches \cite{mcallester1999pac, neyshabur2018pacbayes, dziugaite2017computing, catoni2007pac} achieve non-vacuous bounds for fixed architectures. Compression-based bounds \cite{arora2018compression} exploit parameter redundancy. Zhang et al.\ \cite{zhang2017understanding} showed classical bounds are vacuous for overparameterized networks. Our exit-depth entropy provides an architecture-aware measure capturing adaptive computation.

\textbf{Dynamic Network Theory.} Jazbec et al.\ \cite{jazbec2024fast} provide \textit{conformal} risk control for early exit. However, conformal prediction guarantees coverage rates, fundamentally different from \textit{generalization bounds} characterizing the population-empirical loss gap. Scardapane et al.\ \cite{scardapane2020dynamic} provide empirical analysis but no theoretical guarantees. We provide the first PAC-Bayesian framework for adaptive depth with explicit constants.

\textbf{Information-Theoretic Bounds.} Xu \& Raginsky \cite{xu2017information} established mutual information bounds, Steinke \& Zakynthinou \cite{steinke2020reasoning} refined via conditional MI, with recent improvements \cite{wang2023tighter, he2024information}. All require estimating high-dimensional $I(W;S)$, intractable for deep networks. Our $H(D)$ is a computable, low-dimensional proxy directly actionable via threshold selection.

\section{Discussion and Conclusion}

\textbf{Theoretical Implications.} Exit-depth entropy $H(D)$ measures effective complexity of adaptive-depth networks. Early-exit networks generalize well by \textit{matching computation to input difficulty}: easy inputs exit early, hard inputs use full depth. The sample complexity $\mathcal{O}(\mathbb{E}[D] + H(D)/d)$ rather than $\mathcal{O}(K)$ explains generalization despite more parameters. This suggests adaptive routing acts as an implicit regularizer, dynamically allocating capacity based on input characteristics.

\textbf{Practical Implications.} Our framework enables: (1) threshold selection without validation via Algorithm~\ref{alg:threshold}, valuable in privacy-constrained or few-shot settings; (2) theoretical basis for architecture comparison beyond empirical accuracy; (3) principled policy design by minimizing $H(D)$ while preserving accuracy. These contributions bridge theory and deployment for adaptive inference.

\textbf{Limitations.} (1) The 1.5--4$\times$ tightness gap prevents replacing validation-based selection for high-stakes applications; (2) Experiments on ImageNet-100 and GPT-2-small require validation at larger scales (ImageNet-1K, 7B+ LLMs); (3) Label-dependent routing requires $\epsilon$-approximation analysis with degrading guarantees as $\epsilon K$ increases.

\textbf{Conclusion.} We established the first formal generalization theory for adaptive-depth networks. Exit-depth entropy $H(D)$ is the key complexity measure, with sample complexity $\mathbb{E}[D] + H(D)/d$ rather than $K$. Experiments across 6 architectures and 7 benchmarks validate tightness 1.52--3.87$\times$---substantially better than classical bounds. Bound-guided threshold selection matches validation within 0.1--0.3\%, providing practical value when validation data is limited.

\textbf{Future Work.} Extensions include: spatial/channel-wise adaptive computation (mixture of experts), adversarial robustness bounds, entropy-regularized training, and scaling to production LLMs where early-exit offers significant latency and cost reduction.

\section*{Acknowledgment}
This work was supported by the University of Hong Kong. Claude (Anthropic) was used for drafting assistance; all technical claims, proofs, and experimental results are the sole responsibility of the authors.

\bibliographystyle{IEEEtran}
\bibliography{references}

\end{document}